\documentclass[twoside,11pt]{article}

\usepackage{jmlr2e}
\usepackage[utf8]{inputenc} % allow utf-8 input
\usepackage[T1]{fontenc}    % use 8-bit T1 fonts

\usepackage{url}            % simple URL typesetting
\usepackage{booktabs}       % professional-quality tables
\usepackage{amsfonts}       % blackboard math symbols
\usepackage{nicefrac}       % compact symbols for 1/2, etc.
\usepackage{microtype}      % microtypography
\usepackage{todonotes}

\usepackage{amsmath}
\def\eqref#1{equation~\ref{#1}}
\def\1{\bm{1}}

\def\vh{{\bm{h}}}

\def\vv{{\bm{v}}}

\DeclareMathAlphabet{\mathsfit}{\encodingdefault}{\sfdefault}{m}{sl}
\SetMathAlphabet{\mathsfit}{bold}{\encodingdefault}{\sfdefault}{bx}{n}

\def\gL{{\mathcal{L}}}

\def\gO{{\mathcal{O}}}

\def\sP{{\mathbb{P}}}

\newcommand{\E}{\mathbb{E}}

\newcommand{\Var}{\mathrm{Var}}

\usepackage{bm}
\usepackage[normalem]{ulem}
\usepackage{todonotes}
\usepackage{enumitem}
\usepackage{amsthm}
 
\newtheorem{theorem}{Theorem}
 
\newtheorem{proposition}[theorem]{Proposition} 

\newtheorem{corollary}[theorem]{Corollary}

\usepackage{subfigure}
\usepackage{graphicx}

\ShortHeadings{}{Bias and variance of variational inference}
\firstpageno{1}

\begin{document}

\title{Note on the bias and variance of variational inference}

\author{%
%Chin-Wei Huang & Aaron Courville \\
  Chin-Wei Huang$^{\star}$ $\qquad$ Aaron Courville$^{\star\dagger}$  \\
  $^{\star}$Mila, University of Montreal $\qquad^\dagger$CIFAR Fellow\\
  \texttt{\{chin-wei.huang,aaron.courville\}@umontreal.ca} \\
}
\date{}\maketitle

\begin{abstract}%  
In this note, we study the relationship between the variational gap and the variance of the (log) likelihood ratio. 
We show that the gap can be upper bounded by some form of dispersion measure of the likelihood ratio, which suggests the bias of variational inference can be reduced by making the distribution of the likelihood ratio more concentrated, such as via averaging and variance reduction. 
\end{abstract}

\section{Introduction}
Let $\vv$ and $\vh$ denote the observed and unobserved random variables, following a joint density function $p_\theta(\vv,\vh)$. 
Generally, the log marginal likelihood $\log p_\theta(\vv)=\log\int_\vh p_\theta(\vv,\vh)d\vh$ is not tractable, so the \emph{Maximum likelihood principle} cannot be readily applied to estimate the model parameter $\theta$.  
Instead, one can maximize the \emph{evidence lower bound} (ELBO):
$$\log p_\theta(\vv) = \log \E_{q_\phi(\vh)}\left[ \frac{p_\theta(\vv,\vh)}{q_\phi(\vh)}\right] \geq \E_{q_\phi(\vh)}\left[\log\frac{p_\theta(\vv,\vh)}{q_\phi(\vh)}\right]:=\gL(\theta,\phi)$$
where the inequality becomes an equality if and only if $q(\vh)=p(\vh|\vv)$, since $\log$ is a strictly concave function. 
This way, learning and inference can be jointly achieved, by maximizing $\gL(\theta,\phi)$ wrt $\theta$ and $\phi$, respectively. 

Alternatively, one can maximize another family of lower bounds due to \citet{burda2015importance}:
$$\log p_\theta(\vv) = \log \E_{\vh_j\sim q_\phi(\vh)}\left[ \frac{1}{K}\sum_{j=1}^K\frac{p_\theta(\vv,\vh_j)}{q_\phi(\vh_j)}\right] \geq \E_{\vh_j\sim q_\phi(\vh)}\left[\log\frac{1}{K}\sum_{j=1}^K\frac{p_\theta(\vv,\vh)}{q_\phi(\vh)}\right]:=\gL_K(\theta,\phi)$$
which we call the \emph{importance weighted lower bound} (IWLB). 
Clearly $\gL_1 = \gL$.
An appealing property of this family of lower bounds is that $\gL_K$ is monotonic, i.e. $\gL_M\geq\gL_N$ if $M\geq N$, and can be made arbitrarily close to $\log p_\theta$ provided $K$ is sufficiently large. 

One interpretation for this is that by weighting the samples according to the importance ratio $p/q$, we are effectively correcting or biasing the proposal towards the true posterior $p_\theta(\vh|\vv)$; see \citet{cremer2017reinterpreting} for more details.  
Another interpretation due to \citet{nowozindebiasing} is to view $Y_K:=\log \frac{1}{K}\sum_{j=1}^K\frac{p_\theta(\vv,\vh)}{q_\phi(\vh)}$ as a biased estimator for $\log p_\theta(\vv)$, where the bias is of the order $\gO(K^{-1})$.

We take a different view by looking at the variance, or some notion of dispersion, of $Y_K$. 
We write $X_K:=\exp(Y_K)$ as the average before $\log$ is applied.
The variational gap, $\log\E[X_K]-\E[Y_K]$, is caused by (1) the strict concavity of $\log$, and (2) the dispersion of $X_K$. 
To see this, one can view the expectation $\E[Y_K]$ as the centroid of uncountably many $\log X_K$ weighted by its probability density, which lies below the graph of $\log$. 
By using a larger number of samples, the distribution of $X_K$ becomes more concentrated around its expectation $\E[X_K]=\log p_\theta(\vv)$, pushing the ``centroid'' up to be closer to the graph of $\log$.  
See Figure~\ref{fig:jensen_samples} for an illustration. 

This intuition has been exploited and ideas of correlating the likelihood ratios $X=p/q$ of a joint proposal $q(\vh_1,...,\vh_K)$ have been proposed in \citet{klys2018joint,wu2019differentiable,huang2019hierarchical}.
Even though attempts have been made to establish the connection between $\Var(X)$ and the gap (or bias)  $\log \E[X] - \E[\log X]$, the obtained results are asymptotic and require further assumption on boundedness (such as uniform integrability) of the sequence  $\{X_n\}_{n\geq1}$, which makes the results harder to interpret \footnote{For example, see \citet{klys2018joint,huang2019hierarchical} where they seek to minimize the variance to improve the variational approximation, and \citet{maddison2017filtering,domke2018importance} where they analyze the asymptotic bias by looking at the variance of $X_K$.  }. 
Rather than bounding the asymptotic bias by the variance of $X$, we analyze the non-asymptotic relationship between $\log \E[X]-\E[\log X]$ and the variance of $X$ and $\log X$.
Our finding justifies exploiting the structure of the likelihood ratios of a joint proposal, as anti-correlation among the likelihood ratio serves to further reduce the variance of an average, which we will show in the next section upper bounds the variational gap.

\begin{figure}
\centering
%\begin{minipage}{.75\textwidth}
%    \centering
    \includegraphics[height=8cm,trim={0.0mm 0mm 10mm 0mm},clip]{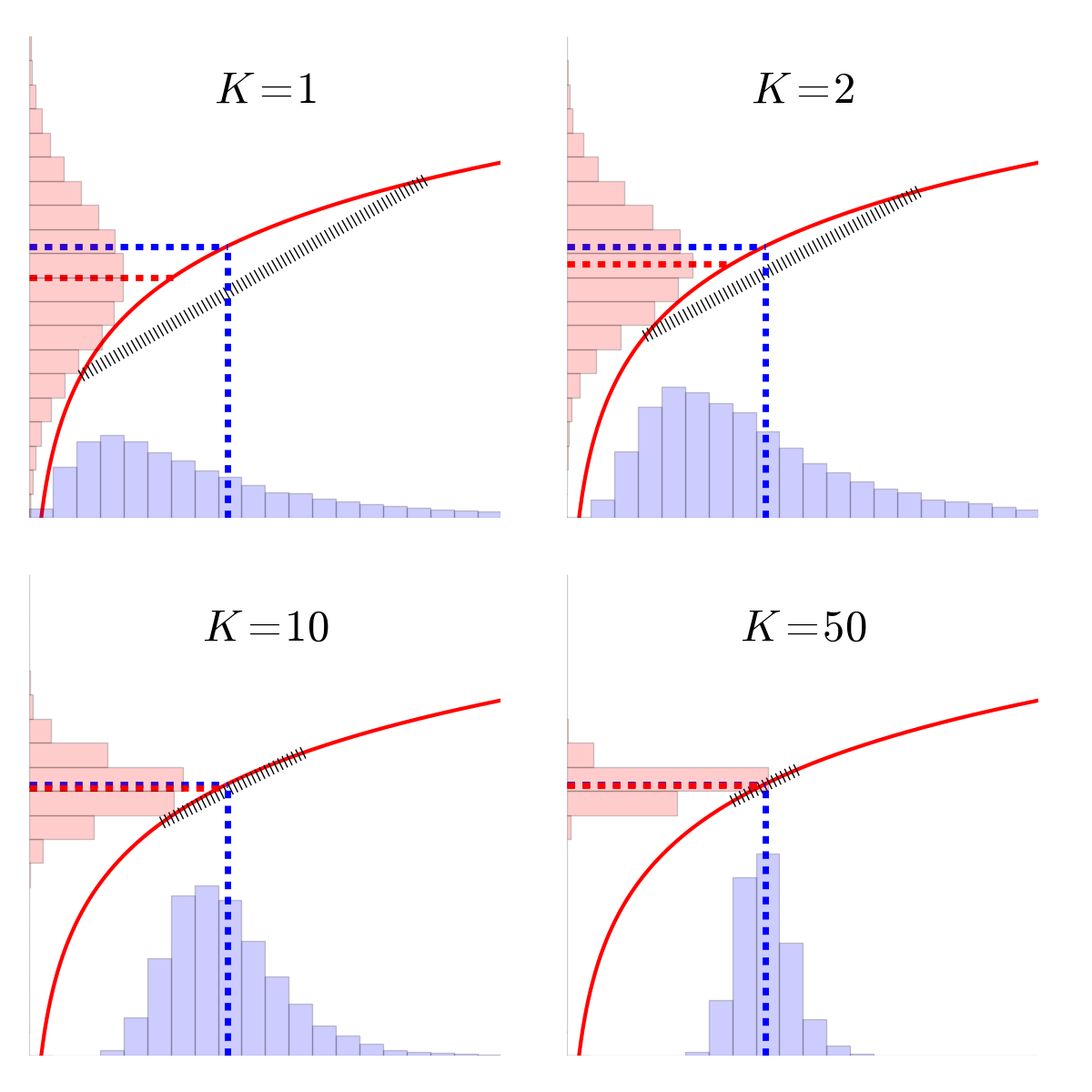}
    \caption{Visualizing the reduction in variaional gap and the concentration of the distribution of the likelihood ratio due to averaging. x-axis: $X_K$. y-axis: $Y_K$. The dotted lines indicate the expected values, and the solid line is the $\log$ function. }
    \label{fig:jensen_samples}
%\end{minipage}
\end{figure}

\section{Bounding the gap via central tendency}

Let $\mu_X$ and $\nu_X$ be the mean and median \footnote{We assume there's a unique median to simplify the analysis.} of a random variable $X$, i.e.
$$\mu_X:=\E[X]\qquad \qquad  \sP(X\geq \nu_X)\wedge\,\sP(X\leq \nu_X)\geq\frac{1}{2}$$
Here we assume $X>0$ is a positive random variable.
One can think of it as $p/q$, or some other unbiased estimate of $p(\vv)$.
By Jensen's inequality, we know $\log\mu_X\geq\mu_Y=\E[\log X]$, where $Y:=\log X$. 
We want to bound the gap $\log\mu_X-\mu_Y$ via some notion of dispersion of $X$ and $Y$. 
Now assume
$\mu_X-\nu_X\leq C_X$ and $\mu_Y-\nu_Y\leq C_Y$. 
Constants $C_X$ and $C_Y$ correspond to the dispersion just mentioned.
For example, the following lemma shows $C_X$ can be taken to be the standard deviation $\sigma_X:=\sqrt{\E[(X-\mu_X)^2]}$:

\begin{proposition}
For $p\geq1$ and $X\in L_p$, then $|\mu_X-\nu_X|\leq ||X-\mu_X||_p$.
\end{proposition}
\begin{proof}
Using the fact that the median minimizes the \emph{mean absolute error} and Jensen's inequality, we have
$$|\mu_X-\nu_X|=|\E[X-\nu_X]|\leq\E[|X-\nu_X|]\leq\E[|X-\mu_X|]\leq||X-\mu_X||_p$$
\end{proof}

Without further assumptions, we can derive a weaker result. 
Since $\log$ is strictly monotonic, $\log\nu_X=\nu_Y$, so we have
$\log\mu_X-\mu_Y = \log\mu_X - \log \nu_X + \nu_Y - \mu_Y$.
Since $\nu_X\geq \mu_X-C_X$, by monotonicity of $\log$, 
$$\log\mu_X-\mu_Y \leq \log\mu_X - \log (\mu_X-C_X) + \nu_Y - \mu_Y$$
which after arrangement gives
$\log (\mu_X-C_X)-\mu_Y \leq  \nu_Y - \mu_Y\leq C_Y$. 
This means if $C_X$ is small enough so that the difference between $\log\mu_X$ and $\log(\mu_X-C_X)$ can be neglected, then the gap of interest is bounded by the dispersion of $Y$, $C_Y$. 

Now, we quantify the error between $\log\mu_X$ and $\log(\mu_X-C_X)$ by the following result:

\begin{figure}
	\centering
	%\begin{minipage}{.75\textwidth}
	%    \centering
	\includegraphics[height=7.5cm,trim={3.0mm 1.5mm 2.5mm 1.0mm},clip]{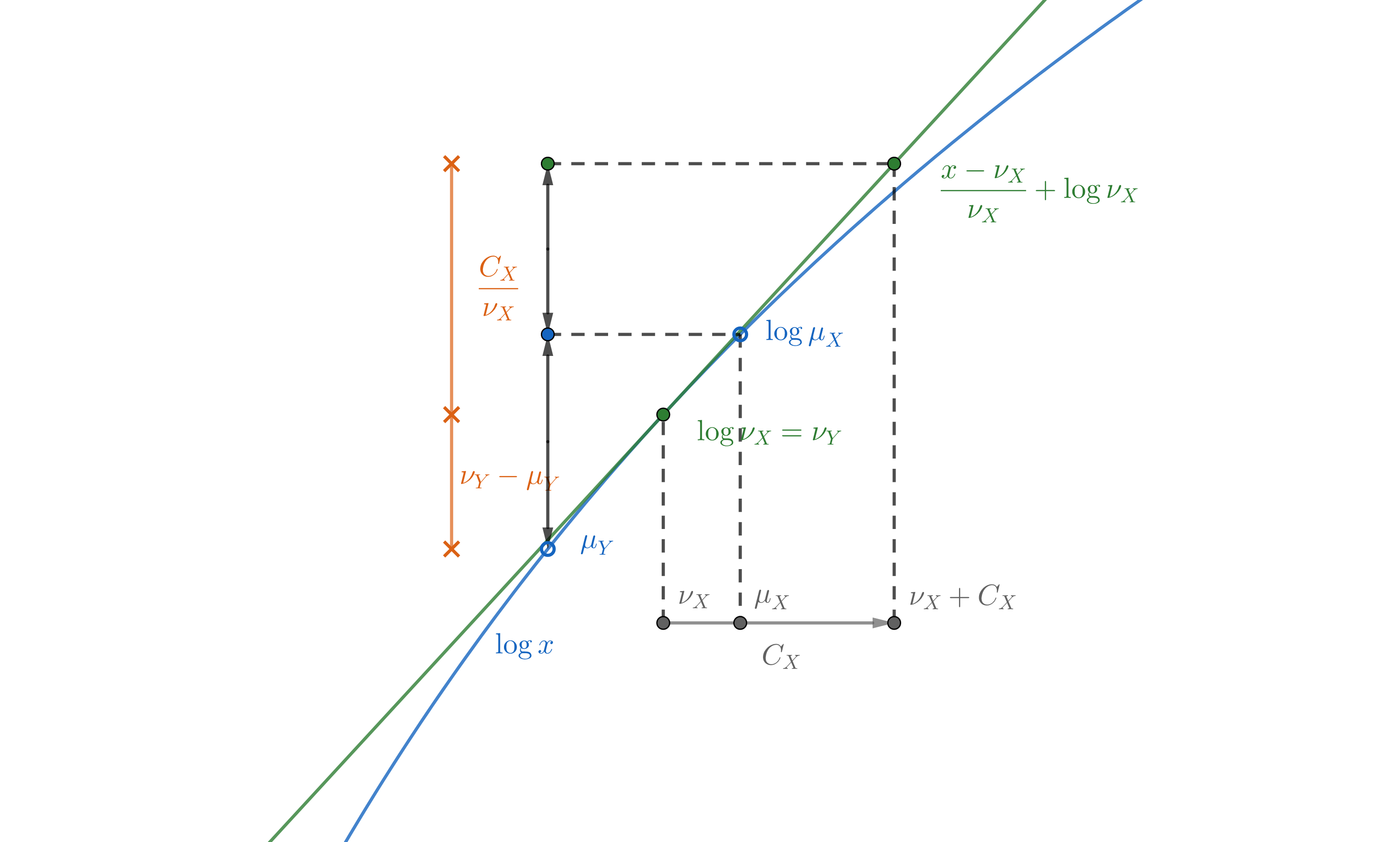}
	\caption{Bounding the variational gap using a linear majorizer (green curve) of $\log$ (blue curve) at $x=\nu_X$.}
	\label{fig:jensen_bound}
	%\end{minipage}
\end{figure}

\begin{proposition}
Let $X>0$ be a positive random variable with $\mu_X=\E[X]$, and $Y=\log X$ with $\mu_Y=\E[Y]=\E[\log X]\leq\log\E[X]=\log\mu_X$. 
Assume 
$$|\mu_X-\nu_X|\leq C_X \qquad { (*) } \qquad \text{ and }  \qquad  |\mu_Y-\nu_Y|\leq C_Y \qquad { (\dagger) } $$
for some constants $C_X,C_Y\geq0$.
If $\mu_X> C_X$, then $$\log \mu_X - \mu_Y \leq \frac{C_X}{\mu_X-C_X} + C_Y$$
\end{proposition}

A visual illustration of the proof is presented in Figure~\ref{fig:jensen_bound}.
The main idea is to use Taylor approximation as a linear upper bound on the $\log$, so that the error in using $\mu_X-C_X$  to approximate $\mu_X$ can be translated to the log scale. 
Hence the additional term $C_X/(\mu_X-C_X)$ is inversely propostional to $\mu_X-C_X$, i.e. the derivative of $\log$, which is the slope of the linear upper bound. 

\begin{proof}
Since $\log$ is a strictly concave function, first-order Taylor approximation (at $\nu_X$) gives a linear upper bound:
$$f(x):= \frac{1}{\nu_X}(x-\nu_X) + \log \nu_X \geq \log(x)$$
By monotonicity of logarithm and $(*)$, 
$\log \mu_X - \mu_Y \leq \log(\nu_X+C_X) - \mu_Y$. 
The logarithm can be bounded from above by the linear upperbound $f(\nu_X+C_X)$, which yields 
$$\log \mu_X - \mu_Y \leq f(\nu_X+C_X) - \mu_Y=\frac{C_X}{\nu_X} + \log \nu_X -\mu_Y$$
Notice that $\log \nu_X = \nu_Y$ (since $\log$ is strictly monotonic), so that we can plug in $(\dagger)$.
Now the premise $\mu_X\geq C_X$ combined with $(*)$ again yields $\frac{1}{\nu_X}\leq \frac{1}{\mu_X-C_X}$, concluding the proof. 
\end{proof}

The main takeaway of the proposition is that if the dispersion of $X$ is sufficiently small, then minimizing the standard deviation of $X$ and $\log X$ amounts to minimizing the gap $\log\E[X]-\E[\log X]$.
We summarize it by the following Corollary:
\begin{corollary}
Let $X>0$ be an unbiased estimator for the marginal likelihood $p(\vv)$, and let $Y=\log X$. 
Denote by $\sigma_X$ and $\sigma_Y$ the standard deviation of $X$ and $Y$, respectively. Then
$$\sigma_X<p(\vv) \quad \Longrightarrow \quad \log p(\vv) - \E[\log X] \leq \frac{\sigma_X}{p(\vv)-\sigma_X} + \sigma_Y$$
\end{corollary}

\acks{We would like to thank Kris Sankaran for proofreading the note.}

% Acknowledgements should go at the end, before appendices and references

%\acks{We would like to acknowledge support for this project from }


\bibliography{main}

\begin{thebibliography}{8}
\providecommand{\natexlab}[1]{#1}
\providecommand{\url}[1]{\texttt{#1}}
\expandafter\ifx\csname urlstyle\endcsname\relax
  \providecommand{\doi}[1]{doi: #1}\else
  \providecommand{\doi}{doi: \begingroup \urlstyle{rm}\Url}\fi

\bibitem[Burda et~al.(2015)Burda, Grosse, and
  Salakhutdinov]{burda2015importance}
Yuri Burda, Roger Grosse, and Ruslan Salakhutdinov.
\newblock Importance weighted autoencoders.
\newblock In \emph{International Conference on Learning Representations}, 2015.

\bibitem[Cremer et~al.(2017)Cremer, Morris, and
  Duvenaud]{cremer2017reinterpreting}
Chris Cremer, Quaid Morris, and David Duvenaud.
\newblock Reinterpreting importance-weighted autoencoders.
\newblock \emph{arXiv preprint arXiv:1704.02916}, 2017.

\bibitem[Domke and Sheldon(2018)]{domke2018importance}
Justin Domke and Daniel~R Sheldon.
\newblock Importance weighting and variational inference.
\newblock In \emph{Advances in Neural Information Processing Systems}, pages
  4470--4479, 2018.

\bibitem[Huang et~al.(2019)Huang, Sankaran, Dhekane, Lacoste, and
  Courville]{huang2019hierarchical}
Chin-Wei Huang, Kris Sankaran, Eeshan Dhekane, Alexandre Lacoste, and Aaron
  Courville.
\newblock Hierarchical importance weighted autoencoders.
\newblock In \emph{International Conference on Machine Learning}, 2019.

\bibitem[Klys et~al.(2018)Klys, Bettencourt, and Duvenaud]{klys2018joint}
Jack Klys, Jesse Bettencourt, and David Duvenaud.
\newblock Joint importance sampling for variational inference.
\newblock 2018.

\bibitem[Maddison et~al.(2017)Maddison, Lawson, Tucker, Heess, Norouzi, Mnih,
  Doucet, and Teh]{maddison2017filtering}
Chris~J Maddison, John Lawson, George Tucker, Nicolas Heess, Mohammad Norouzi,
  Andriy Mnih, Arnaud Doucet, and Yee Teh.
\newblock Filtering variational objectives.
\newblock In \emph{Advances in Neural Information Processing Systems}, pages
  6573--6583, 2017.

\bibitem[Nowozin(2018)]{nowozindebiasing}
Sebastian Nowozin.
\newblock Debiasing evidence approximations: On importance-weighted
  autoencoders and jackknife variational inference.
\newblock In \emph{International Conference on Learning Representations}, 2018.

\bibitem[Wu et~al.(2019)Wu, Goodman, and Ermon]{wu2019differentiable}
Mike Wu, Noah Goodman, and Stefano Ermon.
\newblock Differentiable antithetic sampling for variance reduction in
  stochastic variational inference.
\newblock In \emph{Artificial Intelligence and Statitics (AISTATS)}, 2019.

\end{thebibliography}

\end{document}